\documentclass[UTF8]{llncs}
\usepackage{amssymb}
\usepackage{amsmath}
\usepackage{stmaryrd}
\usepackage{url}
\usepackage[all, knot]{xy}
\usepackage{charter}
\usepackage{color}
\usepackage[english]{babel}
\usepackage[margin=4cm]{geometry}

\renewenvironment{proof} {\textsc{Proof}\quad} {\hfill $\Box$\\}

\newcommand{\BP}{\ensuremath{\mathbf{P}}}
\newcommand{\Act}{\ensuremath{\mathbf{\Sigma}}}

\newcommand{\M}{\mathcal{M}}

\newcommand{\rel}[1]{\stackrel{#1}{\rightarrow}}
\newcommand{\lrel}[1]{\stackrel{#1}{\longrightarrow}}

\newcommand{\lr}[1]{\langle #1 \rangle}

\newcommand{\CONJ}{\ensuremath{\mathtt{UCONJ}}}

\newcommand{\AxTransKU}{\ensuremath{\mathtt{4KU}}}
\newcommand{\AxEucKU}{\ensuremath{\mathtt{5KU}}}
\newcommand{\TRI}{\ensuremath{\mathtt{TRI}}}

\newcommand{\AxTrU}{\ensuremath{\mathtt{TU}}}
\newcommand{\AxTransU}{\ensuremath{\mathtt{4U}}}
\newcommand{\AxEucU}{\ensuremath{\mathtt{5U}}}

\newcommand{\Kv}{\ensuremath{\textit{Kv}}}

\newcommand{\E}{{\mathcal E}}

\newcommand{\DISTU}{\ensuremath{\mathtt{DISTU}}}

\newcommand{\TAUT}{\ensuremath{\mathtt{TAUT}}}

\newcommand{\NECU}{\ensuremath{\mathtt{NECU}}}
\newcommand{\COMPKh}{\ensuremath{\mathtt{COMPKh}}}

\newcommand{\WSKh}{\ensuremath{\mathtt{WSKh}}}

\newcommand{\MP}{\ensuremath{\mathtt{MP}}}

\newcommand{\SUB}{\ensuremath{\mathtt{SUB}}}
\newcommand{\EMP}{\ensuremath{\mathtt{EMP}}}
\newcommand{\COND}{\ensuremath{\mathtt{COND}}}

\newcommand{\NECKh}{\ensuremath{\mathtt{NECKh}}}
\newcommand{\PREKh}{\ensuremath{\mathtt{PREKh}}}
\newcommand{\POSTKh}{\ensuremath{\mathtt{POSTKh}}}

\newcommand{\Kh}{\mathcal{K}h}
\newcommand{\SKh}{\mathbb{SKH}}

\newcommand{\K}{\mathcal{K}}
\renewcommand{\S}{\mathcal{S}}
\renewcommand{\E}{\mathcal{E}}
\newcommand{\V}{\mathcal{V}}
\newcommand{\U}{\mathcal{U}}
\newcommand{\Khp}{\mathcal{K}h^+}
\newcommand{\R}{\mathcal{R}}
\newcommand{\LKh}{\mathbf{L_{Kh}}}
\newcommand{\lra}{\leftrightarrow}

\renewcommand{\phi}{\varphi}

\DeclareSymbolFont{symbolsC}{U}{txsyc}{m}{n}
\DeclareMathSymbol{\strictif}{\mathrel}{symbolsC}{74}

\title{A Logic of Knowing How}
 \author{Yanjing Wang \thanks{The author thanks Frank Veltman for his insightful comments on an earlier version of this paper. The author is also gratful to the support from NSSF key projects 12\&ZD119 and 15AZX020.} \\ \email{y.wang@pku.edu.cn}}
\institute{Department of Philosophy, Peking University}

\begin{document}

\maketitle

\begin{abstract} 
In this paper, we propose a single-agent modal logic framework for reasoning about goal-direct ``knowing how'' based on ideas from linguistics, philosophy, modal logic and automated planning. We first define a modal language to express ``I know how to guarantee $\phi$ given $\psi$'' with a semantics  not based on standard epistemic models but labelled transition systems that represent the agent's knowledge of his own abilities. A sound and complete proof system is given to capture the valid reasoning patterns about ``knowing how'' where the most important axiom suggests its compositional nature.

\end{abstract}

\section[Intro]{Introduction\footnote{To impatient technical readers: this rather philosophical introduction will help you to know \textit{how} the formalism works in the later sections. A bit of philosophy can lead us further. }}
\subsection{Background: beyond ``knowing that''. }
Von Wright and Hinttika laid out the syntactic and semantic foundations of epistemic logic respectively in their seminal works \cite{Wright51} and \cite{Hintikka:kab}. The standard picture of epistemic logic usually consists of: a modal language which can express ``an agent knows that $\phi$''; a Kripke semantics incarnates the slogan ``knowledge (information) as elimination of uncertainty''; a proof system syntactically characterizes a normal modal logic somewhere between $\mathbb{S}4$ and $\mathbb{S}5$ subjective to different opinions about the so-called introspection axioms. Despite the suspicions from philosophers in its early days, the past half-century has witnessed the blossom of this logical investigation of propositional knowledge with applications in  epistemology, theoretical computer science, artificial intelligence, economics, and many other disciplines  besides its birth place of modal logic.\footnote{For an excellent survey of the early history of epistemic logic, see \cite[Chapter 2]{WangRJ11}. For a contemporary comprehensive introduction to its various topics, see \cite{HBEL}.}

However, the large body of research on epistemic logic mainly focuses on propositional knowledge expressed by ``knowing that $\phi$'', despite the fact that in everyday life knowledge is expressed by also ``knowing how'', ``knowing why'', ``knowing what'', ``knowing whether'', and so on (knowing?X below for brevity).  Linguistically, these expressions of knowledge share the common form consisting of the verb ``know'' followed by some embedded questions.\footnote{There is  a cross-lingual fact: such knowing?X sentences become meaningless if the verb ``know'' is replaced by ``believe'', e.g., I believe how to swim. This may shed some shadow on philosophers' usual conception of knowledge in terms of strengthened belief. Linguistically, this phenomenon occurs to many other verbs which can be roughly categorized using factivity, cf., e.g, \cite{Egre08}.} It is natural to assign a high-level uniform truth condition for these knowledge expressions in terms of knowing an answer of the corresponding question \cite{Harrah02}. In fact, in the early days of epistemic logic, Hinttika has elaborate discussions on knowing?X and its relation with questions in terms of first-order modal logic \cite{Hintikka:kab}, which also shapes his later work on \textit{Socratic Epistemology} \cite{Hintikka07}. For example, ``knowing who Frank is'' is rendered as $\exists x \K (\textit{Frank}=x)$ in \cite{Hintikka:kab}. However, partly because of the then-infamous philosophical and technical issues regarding the foundation of first-order modal logic (largely due to Quine), the development of epistemic logics beyond ``knowing that'' was hindered.\footnote{Nevertheless Hintikka addressed some of those issues about first-order modal logic insightfully in the context of epistemic logic, see, e.g., a wonderful survey paper \cite{Hintikka89}. Many of those issues are also elegantly addressed in intensional first-order modal logic cf. e.g., \cite{Fitting98}.} In the seminal work \cite{RAK}, the first-order epistemic logic is just briefly touched without specific discussion of those expressions using different embedded questions. A promising recent approach is based on \textit{inquisitive semantics} where propositions may have both informative content and inquisitive content (cf. e.g.,\cite{CiardelliGR13}). An inquisitive epistemic logic which can handle ``knowing that'' and ``knowing whether'' is proposed in \cite{CiardelliR15}.  

Departing from the linguistically motivated compositional analysis on knowing?X, some researchers took a knowing?X construction as a whole, and introduce a new modality instead of breaking it down by allowing quantifiers, equalities and other logical constants to occur freely in the language \cite{Plaza89:lopc,Hart:1996,wiebeetal:2003}.  For example, ``knowing what a password is'' is rendered by ``\Kv\  \textit{password}'' in \cite{Plaza89:lopc} instead of $\exists x \K\  \textit{password}=x$, where $\Kv$ is the new modality. This move seems promising since by restricting the language we may avoid some philosophical issues of first-order modal logic, retain the decidability of the logic,  and focus on special logical  properties of each particular knowing-?X construction at a high abstraction level. A recent line of work results from this idea \cite{WangF13,WangF14,FanWD14,FWvD15,Xiong14}. Besides the evident non-nomality of the resulting logics,\footnote{For example, \textit{knowing whether} $p\to q$ and knowing whether $p$ together does not entail knowing whether $q$. Likewise, \textit{knowing how} to $p$ and knowing how to $q$ does not entail knowing how to $p\land q$. Moreover, you may not \textit{know why} a tautology is a tautology which contradicts necessitation.} a `signature' technical  difficulty in such an approach is the apparent mismatch of syntax and semantics: the modal language is relatively weak compared to the models which contain enough information to facilitate a reasonable semantics of knowing?X, and this requires new techniques.

\subsection{Knowing how}
Among all the knowing?X expressions, the most discussed one in philosophy and AI is ``knowing how''. Indeed, it sounds the most distant from propositional knowledge (knowledge-that): knowing how to swim seems distinctly different from knowing that it is raining outside. One question that keeps philosophers busy is whether knowledge-how (the knowledge expressed by ``knowing how'') is reducible to knowledge-that. Here philosophers split into two groups: the intellectualists who think knowledge-how is a subspecies of knowledge-that (e.g., \cite{stanley2001knowing}), and the anti-intellectuallists who do not think so  (e.g., \cite{Ryle}). The anti-intellectualism may win your heart at the first glance by equating knowledge-how to certain ability, but the linguistically and logically well-versed intellectualists may have their comebacks at times (think about the previously mentioned interpretation of knowing?X as knowing an answer).\footnote{See \cite{Fantl08} for a survey of the debate. A comprehensive collection of the related papers ($200^+$) can be found at \url{http://philpapers.org/browse/knowledge-how}, edited by John Bengson.} In AI, starting from the early days \cite{Mccarthy69,McCarthy79,Moore85}, people have been studying about representation and reasoning of \textit{procedural knowledge} which is often treated as synonym for knowledge-how in AI, in particular about knowledge-how based on specifiable procedures such as coming out of a maze or winning a game. However, there is no common consensus on how to capture the logic of  ``knowing how'' formally (cf. the excellent surveys \cite{Gochet13,KandA15}). In this paper we presents an attempt to formalize an important kind of ``knowing how'' and lay out its logic foundation, inspired by the aforementioned perspectives of linguistics, philosophy, and AI. 
\medskip

Some  clarifications have to be made before mentioning our ideas and their sources:
\begin{itemize}
\item We will focus on the logic of \textit{goal-direct} ``knowing how'' as Gochet puts it \cite{Gochet13}, such as knowing how to prove a theorem,  how to open the door, how to bake a cake, and how to cure the disease, i.e., linguistically, mainly about knowing how followed by a \textit{achievement verb} or an \textit{accomplishment verb} according to the classification of Vendler \cite{Vendler67}.\footnote{Here knowing how to maintain something or to do an activity (like swimming) are \textit{not} typical examples for our treatment, although we hope our formalism captures some common features shared also by them. As discussed in \cite{Gochet13}, ``knowing how'' plus activities, though more philosophically interesting, is less demanding in logic rendering than others. } On the other hand, we will not talk about the following ``knowing how'': I know how the computer works (explanation), I know how happy she is (degree of emotion), I know how to speak English (rule-direct) and so on. 
\item The goal of this paper is \textit{not} to address the philosophical debate between intellectualism and anti-intellectualism which we did discuss in \cite{LauWang,Lauthesis}. However, to some extent, we are inspired by the ideas from both stands, and combine them in the formal work which may in turn shed new light on this philosophical issue. \footnote{Our logic is more about knowing how than knowledge-how though they are clearly related and often discussed interchangeably in the philosophy literature. The full nature of knowledge-how may not be revealed by the inference structure of the linguistic construction of knowing how.}
\item We focus on the single-agent case without probability, as the first step.
\end{itemize}

\subsection{Basic ideas behind the syntax and semantics}

Different from the cases on ``knowing whether'' and ``knowing what'', there is nothing close to a consensus on what would be the syntax and semantics of the logic of ``knowing how''.  Various attempts were made using Situation Calculus, ATL, or STIT logic to express different versions of ``knowing how'', cf. e.g.,  \cite{Moore85,Morgenstern86,HerzigT06,Broersen08,JamrogaH04,Gochet13}. However, as we mentioned before, we do not favour a compositional analysis using powerful logical languages. Instead, we would like to take the ``knowing how'' construction as a single (and only) modality in our language. It seems natural to introduce a modality $\Kh \phi$ to express the goal-direct ``knowing how to achieve the goal $\phi$''. It sounds similar to ``having the ability to achieve the goal $\phi$'', as many anti-intellectualists would agree. It seems harmless to go one step further as in the AI literature to interpret this type of ``knowing how'' as that the agent \textit{can} achieve  $\phi$. However, it is crucial to note the following problems of such an anti-intelectualistic ability account:  
\begin{enumerate}
\item Knowing how to achieve a goal may not entail that you \textit{can} realize the goal now. For example, as intellectualists would remark, a broken-arm pianist may still know how to play piano even if he cannot play right now, and a chef may still know how to make cakes even when the sugar is run out (cf. e.g., \cite{stanley2001knowing}).
\item Even when you have the ability to win a lottery by luckily buying the right ticket (and indeed win it in the end), it does not mean you know how to win the lottery, since you cannot \textit{guarantee} the result (cf. e.g., \cite{Carr79}). 
\end{enumerate}
To reconcile our intuition about the ability involved in ``knowing how'' and the first problem above, it is observed in \cite{LauWang} that ``knowing how'' expressions in context often come with implicit preconditions.\footnote{Such  conditions are rarely discussed in the philosophical literature of ``knowing how'' with few exceptions such as  \cite{noe2005against}.} For example, when you claim that you know how to go to the city center of Taipei  from the airport, you are talking about what you can do under some implicit preconditions: e.g., the public transportation is still running or there is no strike of the taxi drivers. Likewise, it sounds all right to say that you know how to bake a cake even when you do not have all the ingredients right now: you can do it \textit{given} you have all the ingredients. In our logical language, we make such context-dependent preconditions explicit by introducing the modality $\Kh(\psi, \phi)$ expressing that the agent knows how to achieve $\phi$ given the precondition $\psi$.\footnote{By using the condition, one can be said to know better how to swim than another if he can do it in a more hostile environment (thus weakening the condition) \cite{LauWang}} Actually, we used a similar conditional  knowing what operator in \cite{WangF14} to capture the conditional knowledge such as ``I would know what my password for this website is, given it is 4-digit'' (since I only have one 4-digit password ever).\footnote{Such conditionals are clearly not simple (material) implications and they are closely related to conditional probability and conditional belief (cf. e.g., \cite{TillioHS14}).} In \cite{WangF14}, this conditionaliztion proved to be also useful to encode the potential dynamics of knowledge. We will come back to this at the end of the paper.

Now, to reconcile the intuition of ability with the second problem above, we need to interpret the ability more precisely to exclude the lucky draws. Our main idea comes from \textit{conformant planning} in AI which is exactly about \textit{how} to achieve a goal by a linear plan which can never fail given some initial uncertainty (cf. e.g., \cite{SW98}). For example (taken from \cite{WL12,YLW15}), consider the following map of a floor, and suppose that you know you are at a place marked by $p$ but do not know exactly where you are. Do you know how to reach a safe place (marked by $q$)?
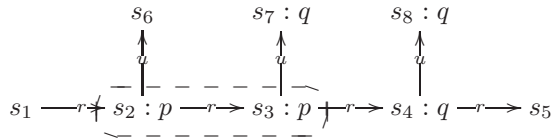
\begin{example}\label{ex.spy}
$$\xymatrix{
&s_6&{{s_7:q}}&{{s_8: q}} &\\
s_1\ar[r]|r& s_2:p\ar[r]|r\ar[u]|u& s_3:p\ar[r]|r\ar[u]|u&{s_4:q}\ar[r]|r\ar[u]|u&s_5
\save "2,2"."2,3"!C="g1"*+[F--:<+20pt>]\frm{}
\restore
}$$
\end{example}
It is not hard to see that there exists a plan to \textit{guarantee} your safety  from any place marked by $p$, which is to move $r$ first then move $u$. On the contrary, the plan $r r$ and the plan $u$ may fail sometimes depending on where you are actually. The locations in the map can be viewed as states of affairs and the labelled directional edges between the states can encode your own ``knowledge'' of the available actions and their effects.\footnote{The agent may have more abilities \textit{de facto} than what he may realize. It is important to make sure the agent can \textit{knowingly} guarantee the goal in terms of the ability he is aware of, cf.  \cite{Mccarthy69,Broersen08,KandA15}.} Intuitively, to know how to achieve $\phi$ requires that you can guarantee $\phi$. Consider the following examples which represent the agent's knowledge about his own abilities. 
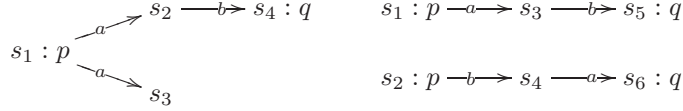
\begin{example}\label{ex.exe}
$$
\xymatrix@R-20pt{
&{s_2}\ar[r]|b&{s_4: q}\\
{s_1:p}\ar[ur]|a\ar[dr]|a\\
&{s_3}
\restore
}
\qquad 
\xymatrix@R-10pt{
{s_1:p}\ar[r]|a&{s_3}\ar[r]|b&{s_5: q}\\
{s_2:p}\ar[r]|b&{s_4}\ar[r]|a&{s_6: q}
\restore
}
$$
\end{example}
The graph on the left denotes that you know you can do $a$ at the $p$-state $s_1$ but you are not sure  what the consequence is: it may lead to either $s_2$ or $s_3$, and the exact outcome is out of your control. Therefore, this action is \textit{non-deterministic} to you. In this case, $ab$ is not a good plan since it may fail to be executable. Thus it sounds unreasonable to claim that you know how to reach $q$ given $p$. Now consider the graph on the right. Let $ab$ and $ba$ be two treatments for the same symptom $p$ depending on the exact cause ($s_1$ or $s_2$). As a doctor, it is indeed true that you can cure the patient (to achieve $q$) if you are told the exact cause. However, responsible as you are, can you say you know how to cure the patient given only the symptom $p$? A wrong treatment may kill the patient. These planning examples suggest the following truth condition for the modal formula $\Kh(\psi, \phi)$ w.r.t.\ a graph-like model  representing the agent's knowledge about his or her abilities (available actions and their possibly non-deterministic effects): 

\begin{quote} There \textit{exists} a sequence $\sigma$ of actions such that from \textit{all} the $\psi$-states in the graph, $\sigma$ will \textit{always} succeed in reaching $\phi$-states. 
\end{quote}
Note that the nesting structure of quantifiers  in the above truth condition is $\exists \forall \forall$.\footnote{In \cite{Brown88},  the author introduced a modality for \textit{can $\phi$} with the following  $\exists\forall$ schema over neighbourhood models: \textit{there is} a relevant cluster of possible worlds (as the outcomes of an action) where $\phi$ is true in \textit{all} of them. } The first $\exists$ fixes a unique sequence, the first $\forall$ checks all the possible states satisfying the condition $\psi$, and the second $\forall$ make sure the goal is guaranteed. 

There are several points to be highlighted: 1. $\exists$ cannot be swapped with the first $\forall$: see the discussion about the second graph in Example \ref{ex.exe}, which amounts to the distinction between \textit{de re} and \textit{de dicto} in the setting of ``knowing how'' (cf. also \cite{Moore85,JamrogaH04,HerzigT06,KandA15} and uniform strategies in imperfect information games); 2. There is no explicit ``knowing that'' in the above truth condition, which differs from the truth conditions advocated by intellectualism \cite{stanley2011know} and the linguistically motivated $\exists x \K \phi(x)$ rendering.\footnote{This also distinguishes this work from our earlier philosophical discussion \cite{LauWang} where intellectualism was defended by giving an $\exists x \K \phi(x)$-like truth condition informally.} On the other hand, the graph model represents the agent's knowledge of his actions and their effects (cf. \cite{Wang15}). 3. The truth condition is based on a Kripke-like model without epistemic relations as in the treatment of (imperfect) procedure information in \cite{Wang15}. As it will become more clear later on, it is not necessary to go for neighbourhood or topological models to accommodate non-normal modal logics if the truth condition of the modality is non-standard (cf. also \cite{KrachtW97,FWvD15,WangF13}); 4. Finally,  our interpretation of ``knowing how'' does not fit the standard scheme ``knowledge as elimination of uncertainty'', and it is not about possible worlds indistinguishable from the ``real world''. The truth of $\Kh(\psi,\phi)$ does not depend on the actual world: it is ``global'' in nature. 

\medskip

In the next section, we will flesh out the above ideas in precise definitions and proofs: first a simple formal language, then the semantics based on the idea of planning, and finally a sound and complete proof system. We hope our formal theory can clarify the above informal ideas further. In the last section, we summarize our novel ideas beyond the standard schema of epistemic logic,  and point out many future directions.

\section{The logic}
\begin{definition}
Given a set of proposition letters $\BP$, the language $\LKh$ is  defined as follows: 
$$\begin{array}{r@{\quad::= \quad}l}
\varphi  &
  \top
         \mid
           p
           \mid \neg \varphi
           \mid (\varphi \land \varphi)
           \mid \Kh (\varphi, \varphi)
\end{array}$$
\noindent where $p\in\BP$. As discussed in the previous section, $\Kh(\psi, \phi)$ expresses that the agent knows how to achieve $\phi$ given $\psi.$ We use the standard abbreviations $\bot, \phi\lor\psi$ and $\phi\to\psi $, and  define $\U\varphi$ as $ \Kh(\neg\varphi,\bot)$. The meaning of $\U$ will become more clear after the semantics is defined.
\end{definition}

\begin{definition} Given the set of proposition letters $\BP$ and a countable non-empty set of action symbols $\Act.$
An \emph{ability map} is essentially a labelled transition system $(\S, \R, \V)$ where:
\begin{itemize}
\item $\S$ is a non-empty set of states;
\item $\R: \Act\to 2^{\S\times \S}$ is a collection of transitions labelled by actions in $\Act$;
\item $\V: S\to 2^\BP$ is a valuation function.
\end{itemize}
We write $s\rel{a}t$ if $(s, t)\in \R(a).$ For a sequence $\sigma=a_1\dots a_n\in\Act^*$, we write $s\rel{\sigma}t$ if there exist $ s_2\dots s_{n}$ such that $s\rel{a_1}s_2\rel{a_2}\cdots \rel{a_{n-1}}s_n\rel{a_n}t$. Note that $\sigma$ can be the empty sequence $\epsilon$ (when $n=0$), and we set $s\rel{\epsilon}s$ for any $s$. Let $\sigma_k$ be the initial segment of $\sigma$ up to $a_k$ for $k\leq |\sigma|$. In particular let $\sigma_0=\epsilon$. We say that $\sigma=a_1\dots a_n$ is \emph{strongly executable} at $s$ if: for any $0\leq k < n$ and any  $t$, $s\rel{\sigma_k}t$ implies that $t$ has at least one $a_{k+1}$-successor.  It is not hard to see that if $\sigma$ is strongly executable at $s$ then it is executable at $s$, i.e., $s\rel{\sigma}t$ for some $t$.
\end{definition}

Note that, according to our above definition,  $ab$ is not strongly executable from $s_1$ in the left-hand-side model of Example~\ref{ex.exe}, since  $s_3$ has no $b$-successor but it can be reached from $s_1$ by $a=(ab)_1$. 

\begin{definition}[Semantics of $\LKh$]
$$\begin{array}{|rcl|}
\hline
\M,s\vDash\top && always\\
\M,s\vDash p &\Leftrightarrow& p\in V(s)\\
\M,s\vDash \neg \varphi &\Leftrightarrow& \M,s\nvDash \varphi\\ 
\M,s\vDash \varphi\land\psi &\Leftrightarrow& \M,s\vDash \varphi \text{ and } \M,s\vDash \psi \\
\M,s\vDash \Kh(\psi,\varphi)&\Leftrightarrow& \text{ there exists a } \sigma\in\Act^* \text{ such that for all $s'$ such that }  \M, s'\vDash \psi: \\
&& \sigma \text{ is strongly executable at $s'$ and for all $t$ such that } s'\rel{\sigma}t, \M, t\vDash \varphi    \\
\hline
\end{array}
$$
\end{definition}

Note that the modality $\Kh$ is \textit{not local} in the sense that its truth  does not depend on the designated state where it is evaluated. Thus it either holds on all the states or none of them.  It is not hard to see that the schema of $\exists\forall\forall$ appears in the truth condition for $\Kh$ where the last $\forall$ actually consists of two parts: the strong executability (there is a $\forall$ in its definition) and the guarantee of the goal. These two together make sure the plan will never fail to achieve $\phi$. It is a simple exercise to see that $\Kh(p, q)$ holds in the model of Example \ref{ex.spy}, but not in the models of Example~\ref{ex.exe}. Moreover, the operator $\U$ defined by $\Kh$ is actually a \textit{universal modality}:\footnote{Note that $\U$ is a very powerful modality in its expressiveness when combined with the standard $\Box$ modality, cf. \cite{GorankoP92}.} 
$$\begin{array}{|rcl|}
\hline

\M,s\vDash \U \varphi&\Leftrightarrow& \text{ for all }t\in \S, \M, t\vDash\varphi   \\
\hline
\end{array}
$$
To see this, check the following:
\begin{small} $$\begin{array}{|rcl|}
\hline
\M,s\vDash \Kh(\neg\psi,\bot)&\Leftrightarrow& \text{ there exists a } \sigma\in\Act^* \text{ such that for every } \M, s'\vDash \neg\psi: \\
&& \sigma \text{ is strongly executable at $s'$ and } \text{ if } s'\rel{\sigma}t \text{ then }\M, t\vDash \bot    \\
&\Leftrightarrow& \text{ there exists a } \sigma\in\Act^* \text{ such that for every } \M, s'\vDash \neg\psi: \\
&& \sigma \text{ is strongly executable at $s'$ and there is no }t \text{ such that } s'\rel{\sigma}t  \\
&\Leftrightarrow& \text{ there exists a } \sigma\in\Act^* \text{ such that for every } \M, s'\vDash \neg\psi: \bot \text{ holds }\\
&\Leftrightarrow& \text{ there exists a } \sigma\in\Act^* \text{ such that there is no }s' \text{ such that } \M, s'\nvDash \psi\\
&\Leftrightarrow&  \text{ for all }t\in \S,  \M, t\vDash \psi\\
\hline
\end{array}
$$
\end{small}
\begin{proposition}\label{prop.sound}
The following are valid: 
\begin{center}
\begin{tabular}{l@{\quad}l@{\qquad }l@{\quad}l}
1& $\U p\land\U (p\to q)\to \U q$ &2 &  $\Kh(p, r)\land\Kh(r, q)\to\Kh(p, q)$\\
3& $\U(p \to q)\to \Kh(p, q)$&4 &$\U p\to p $ \\
5 & $\Kh(p, q)\to\U\Kh(p, q)$&6& $\neg \Kh(p, q)\to\U\neg\Kh(p, q)$
\end{tabular}
\end{center}
%
\end{proposition}
\begin{proof}Since $\U$ is simply a universal modality 1 and 4 are obvious. 3 is due to the fact that $\epsilon$ is allowed as a trivial plan. 5 and 6 are due to the fact that $\Kh$ is global.  The only non-trivial case is 2. Note that if there is a strongly executable sequence  $\sigma$ leading you from any $p$-state to some $r$-state, and there is a strongly executable sequence $\eta$ from $r$-states to $q$-states, then $\sigma\eta$ is strongly executable from any $p$-state and it will make sure that you end up with $q$-states from any $p$-state. 
\end{proof}
\indent The validity of (2) above actually captures the intuitive compositionality of ``knowing how'', as desired. Note that $\Kh(p, q)\land \Kh(p, r)\to \Kh(p, q\land r)$ is not valid, as desired.\\

Based on the above axioms, we propose the following proof system $\SKh$ for $\LKh$ (where $\varphi[\psi\slash p]$ is obtained by uniformly substituting $p$ in $\phi$ by $\psi$): 
\begin{center}
\begin{tabular}{lclc}
\multicolumn{4}{c}{System $\SKh$}\\
\multicolumn{2}{l}{\textbf{Axioms}}&\textbf{Rules}&\\
\TAUT & \text{all axioms of propositional logic}&\MP & $\dfrac{\varphi,\varphi\to\psi}{\psi}$\\
\DISTU & $\U p\land\U (p\to q)\to \U q$&\NECU &$\dfrac{\varphi}{\U\varphi}$\\
\COMPKh & $\Kh(p, r)\land\Kh(r, q)\to\Kh(p, q)$&\SUB &$\dfrac{\varphi(p)}{\varphi[\psi\slash p]}$\\
\EMP &$\U(p \to q)\to \Kh(p, q)$  &\phantom{$\dfrac{\varphi}{[a]\varphi} $}\\
\AxTrU& $\U p\to p $ &\phantom{$\dfrac{\varphi(p)}{\varphi[\psi\slash p]}$} \\
 \AxTransKU& $\Kh(p, q)\to\U\Kh(p, q)$&\phantom{$\dfrac{\varphi}{\Box\varphi}$}\\
 \AxEucKU& $\neg \Kh(p, q)\to\U\neg\Kh(p, q)$&\phantom{$\dfrac{\varphi}{\Box\varphi}$}\\
\end{tabular}
\end{center}

Proposition~\ref{prop.sound} plus some reflection on the usual inference rules should establish the soundness of $\SKh$. For completeness, we first get a taste of the deductive power of $\SKh$ by proving the following formulas which play important roles in the later completeness proof. In the rest of the paper we use $\vdash$ to denote $\vdash_{\SKh}.$
\begin{proposition}
We can derive the following in $\SKh$ (names are given to be used later):
\begin{center}
\begin{tabular}{|c|c|}
\hline
\TRI & $\Kh(p, p)$\\
\WSKh & $ \U(p\to r)\land \U(o\to q)\land \Kh(r, o)\to \Kh(p, q)$\\
\AxTransU&$\U p\to\U\U p$\\
\AxEucU&$\neg\U p\to \U\neg\U p$\\
\COND&$\Kh(\bot, p)$\\
\CONJ &$\U(\varphi\land\psi)\lra(\U\varphi\land\U\psi). $\\
\PREKh & $ \Kh (\Kh(p, q)\land p, q)$. \\
\POSTKh & $\Kh(r, \Kh(p, q)\land p)\to \Kh(r, q) $\\
 \hline
\end{tabular}
\end{center}
\noindent Moreover, the following rule \NECKh\ is admissible: $\vdash\phi \implies \vdash\Kh (\psi, \phi)$.
\end{proposition}
\begin{proof}
\TRI\ is proved by applying  \NECU\ to the tautology $p\to p$, followed by $\EMP$. 
\WSKh\ says if you weaken the goal and strengthen the condition you still know how. It is proved by applying $\EMP$ to $\U(p\to r)$ and $\U(o\to q)$ followed by applying $\COMPKh$ twice. \AxTransU\ and \AxEucU\ are special cases of \AxTransKU\ and \AxEucKU\ respectively since $\U\psi:=\Kh(\neg\psi, \bot).$ Since $\bot\to p$ is a tautology, we can apply $\NECU$ and \EMP\ to obtain $\COND$. \CONJ\ is a  standard exercise for a normal modality. Interestingly, \PREKh\ says that you know how to guarantee $q$ given both $p$ and the fact that you know how to guarantee $q$ given $p$.   It can be proved by distinguishing two cases: $\Kh(p, q)$ and $\neg \Kh(p, q)$, and use \COND\ and \WSKh\ respectively under the help of \NECU. $\POSTKh$ can be proved easily based on \COMPKh\ and \PREKh. It says that you know how to achieve $q$ given $r$ if you know how to achieve a state where you know how to continue to achieve $q$.\footnote{This is an analog of a requirement of the modality \texttt{Can} in \cite{Moore85}.} Finally \NECKh\ is the necessitation rule for $\Kh$ which can be derived by starting with the tautology $\psi\to \phi$ (given $\vdash\phi$) followed by the applications of \NECU\ and \EMP.
%
\end{proof}

\begin{remark}\label{rem.kt}From the above proposition and the system $\SKh$, we see that $\U$ is indeed an $\mathbb{S}5$ modality which can be considered as a version of ``knowing that'': you know that $\phi$ iff it holds on all the relevant possible states under the current restriction of attention (not just the epistemic alternatives to the actual one). The difference is that here the knowledge-that expressed by $\U\phi$ refers to the ``background knowledge'' that you take for granted for now, rather than contingent but epistemically true facts in the standard epistemic logic. Another interesting thing to notice is that $\WSKh$ actually captures an important connection  between ``knowing that'' and ``knowing how'', e.g., you know how to cure a disease if you know that it is of a certain type and you know how to cure this type of the disease in general. We will come back to the relation between ``knowing how'' and ``knowing that'' at the end of the paper.
\end{remark}

It is crucial to establish the following replacement rule to ease the later proofs. 
\begin{proposition}
The replacement of equivalents ($\vdash\phi\lra\psi \implies \vdash\chi\lra \chi [\psi\slash\phi]$)\footnote{Here the substitution can apply to  some (not necessarily all) of the occurrences. } is an admissible rule in $\SKh$. 
\end{proposition}
\begin{proof}
It becomes a standard exercise in modal logic if we can prove that the following two rules are admissible in $\SKh$: $$\vdash\psi\lra\phi \implies\vdash \Kh(\psi, \chi)\lra\Kh(\phi, \chi) , \qquad \vdash\psi\lra\phi \implies\vdash \Kh(\chi, \psi)\lra\Kh(\chi, \phi).$$ Actually we can derive (all the instances of) these two rules as follows (we only show the first one since the second one is similar.): 
$$\begin{array}{lll}
1&\psi\lra\phi& \text{assumed}\\ 
2&\phi\to\psi& 1, \TAUT\\ 
3&\U (\phi\to\psi)& \NECU\\ 
4& \Kh(\phi, \psi)& \MP(\EMP, 3) \\ 
5& \Kh(\psi, \chi)\to\Kh(\phi, \psi)& 4, \TAUT\\
6& \Kh(\psi, \chi)\to (\Kh(\phi, \psi)\land \Kh(\psi, \chi))& 5, \TAUT\\
7& \Kh(\phi, \psi)\land \Kh(\psi, \chi)\to\Kh(\phi, \chi) & \COMPKh, \SUB\\
8& \Kh(\psi, \chi)\to\Kh(\phi, \chi) & \MP (6,7)\\
9& \Kh(\phi, \chi)\to\Kh(\psi, \chi) & \text{symmetric version of 2-8}\\
10& \Kh(\psi, \chi)\lra\Kh(\phi, \chi)&\TAUT
\end{array}$$
\end{proof}
In the rest of the paper we often use the above rule of replacement  implicitly.

\medskip
Here are some notions before we prove the completeness. Given a set of $\LKh$ formulas $\Delta$, let  $\Delta|_{\Kh}$ be the collection of its $\Kh$ formulas:  
$$\Delta|_{\Kh}=\{\chi \mid\chi=\Kh(\psi,\varphi)\in\Delta  \}.$$
Similarly, let $\Delta|_{\neg\Kh}$ be the following collection:  
$$\Delta|_{\neg\Kh}=\{\chi \mid\chi=\neg\Kh(\psi,\varphi)\in\Delta  \}.$$
Now for each maximal consistent set of $\LKh$ formulas we build a canonical model. 
\begin{definition}
Given a maximal consistent set $\Gamma$ w.r.t. $\SKh$, let $\Act_\Gamma=\{ \lr{\psi, \varphi}\mid \Kh(\psi,\varphi)\in \Gamma\}$, 
the canonical model for $\Gamma$ is 
$\M^c_\Gamma=\lr{\S^c_\Gamma, \R^c, \V^c}$ where: 
\begin{itemize}
\item $\S^c_\Gamma=\{\Delta\mid \Delta \text{ is a maximal consistent set w.r.t. $\SKh$ and } \Gamma|_{\Kh}=\Delta|_{\Kh} \}$;
\item $\Delta \lrel{\lr{\psi, \varphi}}_c\Theta $ iff $\Kh(\psi, \varphi)\in \Gamma$,$\psi\in \Delta$, and $\varphi\in\Theta$;
\item $p\in V^c(\Delta)$ iff $p\in \Delta$.
\end{itemize}
Clearly  $\Gamma$ is a state in $\M^c_\Gamma.$ We say that $\Delta\in\S^c_\Gamma$ is a \textit{$\phi$-state} if $\phi\in\Delta.$
\end{definition}

The following two propositions are immediate: 
\begin{proposition}\label{prop.equiv}
For any $\Delta, \Delta'$ in $\S^c_\Gamma$, any $\Kh(\psi, \varphi)\in \LKh$, $\Kh(\psi, \varphi)\in \Delta$ iff $\Kh(\psi, \varphi)\in \Delta'$ iff $\Kh(\psi, \varphi)\in\Gamma. $ 
\end{proposition}

\begin{proposition}\label{prop.reachall}
If $\Delta\lrel{\lr{\psi, \varphi}}\Theta$ for some $\Delta,\Theta\in \S^c_\Gamma$ then $\Delta\rel{\lr{\psi, \phi}}\Theta'$ for any $\Theta'$ such that $\phi\in\Theta'.$ 
\end{proposition}

Now we prove a crucial proposition to be used later. 
\begin{proposition}\label{prop.u}
If $\varphi\in \Delta$ for all $\Delta\in \S^c_\Gamma$ then $\U\varphi\in \Delta$ for all $\Delta \in \S^c_\Gamma$.   
\end{proposition}
\begin{proof}
Suppose $\varphi\in \Delta$ for all $\Delta\in \S^c_\Gamma$, then by the definition of $\S^c_\Gamma$, $\neg\varphi$ is not consistent with $\Gamma|_{\Kh}\cup\Gamma|_{\neg\Kh}$, for otherwise $\Gamma|_{\Kh}\cup\Gamma|_{\neg\Kh}\cup\{\neg \phi\}$ can be extended into a maximal consistent set in $\S^c_\Gamma$ due to a standard Lindenbaum-like argument. Thus there are $\Kh(\psi_1, \varphi_1)$, \dots, $\Kh(\psi_k, \varphi_k) \in \Gamma|_{\Kh}$ and $\neg \Kh(\psi'_1, \varphi'_1)$, \dots, $\neg\Kh(\psi'_l, \varphi'_l) \in \Gamma|_{\neg \Kh}$ such that 
$$\vdash \bigwedge_{1\leq i\leq k}\Kh(\psi_i, \varphi_i)\land \bigwedge_{1\leq j\leq l}\neg \Kh(\psi'_j, \varphi'_j)\to \varphi.$$
By $\NECU$,  $$\vdash \U(\bigwedge_{1\leq i\leq k}\Kh(\psi_i, \varphi_i)\land \bigwedge_{1\leq j\leq l}\neg \Kh(\psi'_j, \varphi'_j)\to \varphi).$$ By $\DISTU$  we have: $$\vdash\U(\bigwedge_{1\leq i\leq k}\Kh(\psi_i, \varphi_i)\land \bigwedge_{1\leq j\leq l}\neg \Kh(\psi'_j, \varphi'_j)) \to \U\varphi.$$   Since $\Kh(\psi_1, \varphi_1)$, \dots, $\Kh(\psi_k, \varphi_k) \in \Gamma|_{\Kh}$, we have $\U\Kh(\psi_1, \varphi_1)$, \dots, $\U\Kh(\psi_k, \varphi_k) \in \Gamma$ due to $\AxTransU$ and the fact that $\Gamma$ is a maximal consistent set. Similarly, we have $\U\neg \Kh(\psi'_1, \varphi'_1)$, \dots, $\U\neg\Kh(\psi'_j, \varphi'_j) \in \Gamma$ due to $\AxEucU$. Now due to a slight generalization of \CONJ, we have: $$\U(\bigwedge_{1\leq i\leq k}\Kh(\psi_i, \varphi_i)\land \bigwedge_{1\leq j\leq l}\neg \Kh(\psi'_j, \varphi'_j))\in \Gamma.$$  Now it is immediate that $\U\varphi\in\Gamma$. Due to Proposition~\ref{prop.equiv}, $\U\varphi\in\Delta$ for all $\Delta\in \S^c_\Gamma.$
\end{proof}
Now we are ready to establish another key proposition for the truth lemma. 
\begin{proposition}\label{prop.imp}
Suppose that there are $\psi',\phi' \in \LKh$ such that for each $\psi$-state $\Delta\in \S^c_\Gamma$ we have $\Delta\lrel{\lr{\psi', \varphi'}}\Theta$ for some $\Theta \in \S^c_\Gamma$, then $\U(\psi\to\psi')\in \Delta$ for all $\Delta\in \S^c_\Gamma.$  
\end{proposition}
\begin{proof}
Suppose that every $\psi$-state has an outgoing $\lr{\psi', \varphi'}$-transition,  then by the definition of $\R^c$, $\psi'$ is in all the $\psi$-states. For each $\Delta$, either $\psi\not\in \Delta$ or $\psi\in\Delta$ thus $\psi'\in \Delta$. Now by the fact that $\Delta$ is maximally consistent it is not hard to show $\psi\to\psi'\in\Delta$ in both cases.  By Proposition~\ref{prop.u}, $\U(\psi\to\psi') \in \Delta$ for all $\Delta\in\S^c_\Gamma.$
\end{proof}
Now we are ready to prove the truth lemma. 
\begin{lemma}[Truth lemma]
For any $\varphi\in\Gamma: \M_\Gamma^c, \Delta\vDash\varphi \iff \varphi\in\Delta$
\end{lemma}
\begin{proof}Boolean cases are trivial and we only focus on the case of $\Kh(\psi, \varphi)$:. 

\noindent $\implies:$ If $\M^c_\Gamma, \Delta\vDash \Kh(\psi, \varphi)$, then according to the semantics,  there exists a (possibly empty) sequence  $\sigma\in{\Act_\Gamma}^*$  such that for every  $ \Delta'\vDash \psi$: $\sigma$ is strongly executable at $\Delta'$ and if $\Delta'\rel{\sigma}\Delta''$ then $\M^c_\Gamma, \Delta''\vDash\phi$. Due to the construction of $\R^c$, there are $\Kh(\psi_1, \varphi_1)$, \dots, $\Kh(\psi_k, \varphi_k)$ in $\Gamma$ such that $\sigma=\lr{\psi_1, \varphi_1}\dots \lr{\psi_k, \varphi_k}$ for some $k\geq 0.$ If there is no $\psi$-state, then by IH, $\neg\psi\in \Theta$ for all $\Theta\in\S^c_\Gamma$. By Proposition~\ref{prop.u}, $\U\neg\psi\in\Delta$, i.e., $\Kh(\psi, \bot)\in\Delta$. Since $\bot\to\phi$ and $\psi\to \psi$ are tautologies, by $\NECU$, $\U(\bot\to \phi)$ and $\U(\psi\to \psi)$ are also in $\Delta$. Then by \SUB\ and \WSKh\ $\Kh(\psi, \phi)\in \Delta.$  In the following we suppose that there exists some $\psi$-state and call this assumption ($\circ$). There are two cases to be considered: 

\begin{itemize}
\item Suppose that $k=0$, i.e, $\sigma=\epsilon$: By the semantics, we have for any $\Theta\in \S^c_\Gamma:  \Theta\vDash\psi\to \varphi$, i.e., $\Theta\nvDash \psi$ or $\Theta\vDash\varphi$.  By IH, $\psi\not\in\Theta$ or $\phi\in\Theta$. Since $\Theta$ is maximally consistent, $\psi\to\varphi\in \Theta$ for all $\Theta\in \S^c_\Gamma$. By Proposition~\ref{prop.u}, $\U(\psi\to\varphi)\in\Theta$ for all $\Theta\in \S^c_\Gamma$. By \SUB\ and \EMP,  we have $\Kh(\psi, \phi)\in\Theta$ for all $\Theta\in\S^c_\Gamma$, in particular, $\Kh(\psi, \phi)\in\Delta.$
\item Suppose $k>0$, recall that $\sigma_m$ is the initial segment of $\sigma$ up to $\lr{\psi_m, \varphi_m}$.  We first prove the following claim ($\star$): 

\quotation{ \noindent For any $m\in [1, k]$: (1) $\Kh(\psi, \varphi_m)\in\Gamma $, and  (2) each $\varphi_m$-state is  reached via $\sigma_m$ from some $\psi$-state.}
\medskip

\begin{itemize}
\item $m=1:$ Due to the semantics of $\Kh$, each $\psi$-state has an outgoing $\lr{\psi_1, \varphi_1}$-transition. By Proposition~\ref{prop.reachall}, each $\psi$-state is connected with all the $\varphi_1$ states by $\lr{\psi_1, \varphi_1}$-transitions, which proves (2). By Proposition~\ref{prop.imp}, $\U(\psi\to\psi_1)\in \Gamma$. By the definition of $\R^c$ it is clear that $\Kh(\psi_1, \varphi_1)\in\Gamma$. Now by \SUB\ and  \WSKh, $\Kh(\psi, \varphi_1)\in \Gamma$.  
\item Suppose that the claim holds for $m=n-1$ then we prove that it holds for $m=n$ as well. By IH for the above claim ($\star$), we have (i) $\Kh(\psi,\varphi_{n-1})\in\Gamma$ and (ii) all the $\varphi_{n-1}$-states are reached from some $\psi$-state by $\sigma_{n-1}$. Since $\sigma$ is a witness of the truth of $\Kh(\psi, \phi)$, $\sigma$ is strongly executable on each $\psi$-state. Now due to (ii) and the strong executability of $\sigma$, we have (iii): each $\varphi_{n-1}$-state has some $\lr{\psi_n, \phi_n}$-successor. By Proposition~\ref{prop.imp} we have $\U(\varphi_{n-1}\to \psi_n)\in\Gamma$. By the definition of $\R^c$, $\Kh(\psi_n, \varphi_n)\in \Gamma$, thus by \WSKh\ we have $\Kh(\phi_{n-1}, \phi_n)\in\Gamma$. Due to (i) and \COMPKh, $\Kh(\psi, \phi_n)\in\Gamma.$  Now for (2) of the claim: By (iii) and the definition of Proposition~\ref{prop.reachall}, each $\varphi_n$-state is reached from some $\phi_{n-1}$-state via $\lr{\psi_{n}, \phi_n}$. Thus based on (ii) again, we have  each $\varphi_n$-state is reached from some $\psi$-state. 
\end{itemize}
Now Claim $(\star)$ is proved.  Let $m=k$, we have ($1_k$) $\Kh(\psi, \phi_k)\in \Gamma$ and ($2_k$) each $\varphi_k$-state is reached via $\sigma_k=\sigma$ from some $\psi$-state (under the assumption $(\circ)$). Now since $\sigma$ witnesses the truth of $\Kh(\psi, \phi)$, $\M^c_\Gamma, \Delta'\vDash\phi$ for every $\varphi_k$-state $\Delta'$. By IH of the main structural induction over formulas, $\phi\in \Delta'$ for every $\Delta'$ such that $\phi_k\in \Delta'$. Thus it is not hard to see that $\phi_k\to\phi$ is in every state of $\S^c_\Gamma$, for otherwise there is a state $\Delta'$ such that $\phi_k, \neg \phi\in\Delta'$.  By Proposition~\ref{prop.u}, $\U(\phi_k\to\phi)\in\Gamma$.  Thus $\Kh(\psi, \phi)\in\Gamma$ due to ($1_k$),  \COMPKh\ and \SUB. Therefore $\Kh(\psi,\phi)\in\Delta$ due to Proposition~\ref{prop.equiv}. 
\end{itemize}
This completes the proof of $\Kh(\psi,\phi)\in\Delta$ if $\M^c_\Gamma, \Delta\vDash\Kh(\psi, \phi)$. 

\medskip

\noindent Now for the other way around:\\
 $\Longleftarrow:$ Suppose  that $\Kh(\psi, \varphi)\in\Delta$, i.e., $\Kh(\psi, \phi)$ is in all the states of $\M^c_\Gamma$ by Proposition~\ref{prop.equiv}, we need to show  that $\M^c_\Gamma, \Delta\vDash \Kh(\psi, \varphi)$. There are three cases to be considered. 
\begin{itemize}
\item There is no $\Theta$ such that $\psi\in\Theta$. By IH, there is no $\Theta$ such that $\M^c_\Gamma,  \Theta\vDash \psi$ then $\M^c_\Gamma, \Delta\vDash \Kh(\psi, \varphi)$ trivially holds (by letting $\sigma=\epsilon$ as the witness). 
\item There are $\Theta, \Theta'$ such that $\psi\in\Theta$ and $\varphi\in \Theta'$. Then by IH, we have  $\M^c_\Gamma, \Theta\vDash \psi$ and $\M^c_\Gamma, \Theta' \vDash \varphi$ for such $\Theta$ and $\Theta'$. Then by the construction of $\R^c$ and IH again we know that $\lr{\psi, \varphi}\in\Act_\Gamma$ is strongly executable, and it will take you to states where $\phi$ is true  from any state where $\psi$ is true.   
\item There is some  $\Theta$ such that $\psi\in\Theta$ but no $\Theta$ such that  $\varphi\in \Theta$. Then it is clear that $\neg\varphi\in\Theta$ for all $\Theta\in\S^c_\Gamma$. By Proposition~\ref{prop.u}, $\U\neg\varphi\in\Theta$ for all $\Theta\in\S^c_\Gamma$. Now we have $\Kh(\varphi,\bot)$ and $\Kh(\psi,\varphi)\in\Theta$ thus by $\COMPKh$ $\Kh(\psi, \bot)\in\Theta$ namely $\U\neg \psi\in \Theta$. By $\AxTrU$, $\neg\psi\in\Theta$ for all the $\Theta \in\S^c_\Gamma$. However, this is contradictory to the assumption that  $\psi\in \Theta $ for some $\Theta\in\S^c_\Gamma.$

\end{itemize}

\end{proof}
Now due to a standard Lindenbaum-like argument, each $\SKh$-consistent set of formulas can be extended to a maximal consistent set $\Gamma$. Due to the truth lemma, $\M^c_\Gamma, \Gamma\vDash\Gamma.$ The completeness of $\SKh$ follows immediately. 
\begin{theorem}
$\SKh$ is sound and strongly complete w.r.t. the class of all models. 
\end{theorem}

\section{Conclusions and future work}
In this paper, we propose and study a modal logic of goal-direct ``knowing how''. The highlights of our framework are summarized below with connections to our earlier ideas on non-standard epistemic logics: 
\begin{itemize}
\item The ``knowing how'' construction is treated as a whole similar to our works on ``knowing whether'' and ``knowing what'' \cite{WangF13,FWvD15}. We would like to keep the language neat. 
\item Semantically, ``knowing how'' is treated as a special conditional:  \textit{being able} to guarantee a goal given a precondition, partly inspired by the conditionalization in \cite{WangF14}. 
\item The \textit{ability} involved is further interpreted as having a plan that never fails to achieve the goal under the precondition, inspired by the work on conformant planning  \cite{YLW15} where we uses the epistemic PDL language to encode the planning problem.
\item The semantics is based on labelled transition systems representing the agent's knowledge of his own abilities, inspired by the framework experimented in \cite{Wang15}. 
\item Compared to the standard semantic schema of knowledge-that: true in \textit{all} indistinguishable alternatives, our work has a more existential flavour: knowing how as having at least \textit{one} good plan. Our modal operator is not local to the indistinguishable alternatives but it is  about all the possible states even when they are distinguishable from the current world. Thus a cook can still be said to know how to cook a certain dish even if he knows that the ingredients are not available right now. 
\end{itemize} 
There are a lot more to explore. We conjecture the logic is decidable and leave the model-theoretical issues to the full version of this paper. Moreover, it is a natural extension to introduce the standard knowing-that operator $\K$ into the language and correspondingly add a set $\E\subseteq \S$ in the model to capture the agent's \textit{local} epistemic alternatives. Then we can define the local version of ``knowing how'' $\Kh \phi$ as $\K\psi\land \Kh(\psi, \phi)$ for some $\psi$. Other obvious next steps include probabilistic and multi-agent versions of $\Kh$. It also makes good sense to consider group notions of ``knowing how'' which may bring it closer to the framework of ATEL where a group of agents may achieve a lot more together (cf.  \cite{KandA15}). More generally, we may consider program-based ``knowing how'' where conditional plans and iterated plans are allowed, which can be used to \textit{maintain} a goal. It is also interesting to add the dynamic operators to the picture, i.e., the public announcements $[\phi]$. In particular,  it is interesting to see how new knowledge-how is obtained by learning new knowledge-that e.g., $\Kh(p, q) \to  [p](\U p\land \Kh(\top, q))$ may be a desired valid formula.\footnote{Note that $\LKh$ may not have the enough pre-encoding power for announcements in itself, similar to the case of PALC discussed in \cite{LCC}.  In particular, $[\chi]\Kh(\psi, \phi)\lra \Kh([\chi]\psi, [\chi]\phi)$ may not be valid due to the lack of control in the syntax for the intermediate stages of the execution of a plan.} 

There are also interesting philosophical questions related to our formal theory. For example,  a new kind of logical omniscience may occur: if there is indeed a good plan to achieve $\phi$ according to the agent's abilities then he knows how to achieve $\phi$. To the taste of philosophers, maybe an empty plan is not acceptable to witness knowledge-how, e.g., people would not say I know how to digest (by doing nothing). We can define a stronger modality $\Khp (\psi,\varphi)$ as $\Kh(\psi,\varphi)\land \neg \U(\psi\to\varphi)$ to rule out such cases.\footnote{The distinction between $\Kh$ and $\Khp$ is similar to the distinction between STIT and deliberative STIT.}  Note that although $\U$ is definable by $\Kh$ in our setting, it does not have the philosophical implication that knowledge-that is actually a subspecies of knowledge-how, as strong anti-intellectulism would argue. 
 Nevertheless, our axioms do tell us something about the interactions between ``knowing how'' and ``knowing that'',  e.g.,  $\WSKh$ says some (global) knowledge-that may let us know better how to reach our goal.

\bibliographystyle{splncs}
\bibliography{questions}
\end{document}